\documentclass{article}

\usepackage[nonatbib, final]{neurips_2019}

\usepackage[utf8]{inputenc} 
\usepackage{hyperref}       
\usepackage{url}            
\usepackage{booktabs}       
\usepackage{amsfonts}       
\usepackage{amsmath}
\usepackage{graphicx,wrapfig,lipsum}
\usepackage{amsfonts}
\usepackage{booktabs}
\usepackage{empheq}
\usepackage{amsthm}
\usepackage{caption}
\usepackage{subcaption}
\usepackage{mathrsfs}  
\usepackage{amssymb}
\usepackage{tikz}
\usetikzlibrary{calc,positioning}

\theoremstyle{plain}
\newtheorem{thm}{Theorem}[section]
\newtheorem{proposition}[thm]{Proposition}

\theoremstyle{definition}
\newtheorem{definition}[thm]{Definition}
\newtheorem{remark}[thm]{Remark}
\newtheorem{example}[thm]{Example}

\newcommand{\nomath}[1]{\ifmmode
\mathrm{#1}%
\else
#1%
\fi}

\newcommand{\TODO}[1][]{\textcolor{red}{\ifx&#1&
\nomath{TODO}
\else
\nomath{TODO: #1}
\fi}}

\def\layersep{2.5cm}

\newcommand{\corners}{6pt}

\title{Deep Signature Transforms}

\newlength{\imanol}
\newcommand{\equallength}[1]{\makebox[\imanol][c]{#1}}

\author{ 
	\equallength{Patric Bonnier$^{1, }$\thanks{Equal contribution.}}
	\And
	\equallength{Patrick Kidger$^{1, 2, }$\footnotemark[1]}
	\And
	\equallength{Imanol Perez Arribas$^{1, 2, }$\footnotemark[1]}
	\And
	\equallength{Cristopher Salvi$^{1, 2, }$\footnotemark[1]}
	\And
	\equallength{Terry Lyons$^{1, 2}$}
	\AND \\[-12pt]
	\null$^1$ Mathematical Institute, University of Oxford \\
	\null$^2$ The Alan Turing Institute, British Library \\
	\texttt{\{bonnier, kidger, perez, salvi, tlyons\}@\hspace{0.1pt}maths.ox.ac.uk}
}

\begin{document}
	\maketitle
	\begin{abstract}
		The signature is an infinite graded sequence of statistics known to characterise a stream of data up to a negligible equivalence class. It is a transform which has previously been treated as a fixed feature transformation, on top of which a model may be built. We propose a novel approach which combines the advantages of the signature transform with modern deep learning frameworks. By learning an augmentation of the stream prior to the signature transform, the terms of the signature may be selected in a data-dependent way. More generally, we describe how the signature transform may be used as a layer anywhere within a neural network. In this context it may be interpreted as a pooling operation. We present the results of empirical experiments to back up the theoretical justification. Code available at \texttt{github.com/patrick-kidger/Deep-Signature-Transforms}.
	\end{abstract}
	\section{Introduction}
	\subsection{What is the signature transform?}
	When data is ordered sequentially then it comes with a natural path-like structure: the data may be thought of as a discretisation of a path $X \colon [0, 1] \to V$, where $V$ is some Banach space. In practice we shall always take $V = \mathbb R^d$ for some $d \in \mathbb N$. For example the changing air pressure at a particular location may be thought of as a path in $\mathbb R$; the motion of a pen on paper may be thought of as a path in $ \mathbb R^2$; the changes within financial markets may be thought of as a path in $\mathbb R^d$, with $d$ potentially very large.
	
	Given a path, we may define its \emph{signature}, which is a collection of statistics of the path. The map from a path to its signature is called the \emph{signature transform}.
	
	\begin{definition}\label{sig-definition}
		Let $\mathbf x = (x_1, \ldots, x_n)$, where $x_i \in \mathbb R^d$
		. Let $f =(f_1, \ldots, f_d) \colon [0, 1] \to \mathbb R^d$ be continuous, such that $f(\tfrac{i -1}{n - 1}) = x_i$, and linear on the intervals in between. Then the signature of $\mathbf x$ is defined as the collection of iterated integrals\footnote{For clarity here we have used more widely-understood notation. The definition of the signature transform is usually written in an equivalent but alternate manner using the notation of stochastic calculus; see Definition \ref{def:sigstochasticnotation} in Appendix \ref{appendix:sigprop}.}
		\begin{equation*}
		\mathrm{Sig}(\mathbf x) = \left(\left( \underset{0 < t_1 < \cdots < t_k < 1}{\int\cdots\int} \prod_{j = 1}^k \frac{\mathrm d f_{i_j}}{\mathrm dt}(t_j) \mathrm dt_1 \cdots \mathrm dt_k \right)_{\!\!1 \leq i_1, \ldots, i_k \leq d}\right)_{\!\!k\geq 0}.
		\end{equation*}
	\end{definition}	
	We shall often use the term \emph{signature} to refer to both a path's signature and the signature transform. Other texts sometimes use the term \emph{path signature} in a similar manner.
	
	We refer the reader to \cite{primer2016} for a primer on the use of the signature in machine learning. A brief overview of its key properties may be found in Appendix \ref{appendix:sigprop}, along with associated references.
	
	In short, the signature of a path determines the path essentially uniquely, and does so in an efficient, computable way. Furthermore, the signature is rich enough that every continuous function of the path may be approximated arbitrarily well by a linear function of its signature; it may be thought of as a `universal nonlinearity'. Taken together these properties make the signature an attractive tool for machine learning. The most simple way to use the signature is as feature transformation, as it may often be simpler to learn a function of the signature than of the original path.
	
	Originally introduced and studied by Chen in \cite{Chen54, Chen57, Chen58}, the signature has seen use in finance \cite{lyonsfinance2014, gyurko2014, lyons2019nonparametric, lyons2019model, kalsi2019optimal}, rough path theory \cite{lyons1998differential, FritzVictoir10} and machine learning \cite{yang2015chinese, xie2018learning, yang2016dropsample, yang2016deepwriterid, li2017lpsnet, yang2017leveraging, yang2016rotation, kiraly2016kernels, chevyrev2018signature}.
	
	\subsection{Comparison to the Fourier transform}
	The signature transform is most closely analogous to the Fourier transform.
	
	The fundamental difference between the signature transform and classical signal transforms such as Fourier transforms and wavelets is that the latter are used to model a curve as a linear combination in a functional basis. The signature does not try to model or parameterise the curve itself, but instead provides a basis for functions on the space of curves.
	
	For example, regularly seeing the sequence: phone call, trade, price movement in the stream of office data monitoring a trader might be an indication of insider trading. Such occurrences are straightforward to detect by via a linear regression composed with the signature transform. Modelling this signal using Fourier series or wavelets would be much more expensive: linearity of these transforms imply that each channel must be resolved accurately enough to see the order of events.
	
	From a signal processing perspective, the signature can be thought of as a filter which is invariant to resampling of the input signal. (See Proposition \ref{prop:invariance} in Appendix \ref{appendix:sigprop}).

\subsection{Use of the signature transform in machine learning}
	
	The signature is an infinite sequence, so in practice some finite collection of terms must be selected. Since the magnitude of the terms exhibit factorial decay, see Proposition \ref{prop:factorialdecay} in Appendix \ref{appendix:sigprop}, it is usual \cite{lyons2014rough} to simply choose the first $N$ terms of this sequence, which will typically be the largest terms
. These first $N$ terms are called the \emph{signature of depth $N$} or the \emph{truncated signature of depth $N$}, and the corresponding transform is denoted $\mathrm{Sig}^N$. But if the function to be learned depended nontrivially on the higher degree terms, then crucial information has nonetheless been lost.
	
	This may be remedied. Apply a pointwise augmentation to the original stream of data before taking the signature. Then the first $N$ terms of the signature may better encode the necessary information \cite{kiraly2016kernels, chevyrev2018signature}. Explicitly, let $\Phi \colon \mathbb R^d \to \mathbb R^e$ be fixed; one could ensure that information is not lost by taking $\Phi(x) = (x, \varphi(x))$ for some $\varphi$. Then rather than taking the signature of $\mathbf x = (x_1, \ldots, x_n)$, where $x_i \in \mathbb R^d$, instead take the signature of $\Phi(\mathbf x) = (\Phi(x_1), \ldots, \Phi(x_n))$. In this way one may capture higher order information from the stream in the lower degree terms of the signature.
	
	\subsection{Our work}
	But how should this augmentation $\Phi$ be chosen? Previous work has fixed it arbitrarily, or experimented with several options before choosing one \cite{kiraly2016kernels, chevyrev2018signature}. Observe that in each case the map $\mathbf x \mapsto \mathrm{Sig}^N(\Phi(\mathbf x))$ is still ultimately just a feature transformation on top of which a model is built. Our more general approach is to allow the selection of $\Phi$ to be data-dependent, by having it be learned; in particular it may be a neural network. Furthermore there is no reason it should necessarily operate pointwise, nor (since it is now learned) need it be of the form $ (x, \varphi(x)) $. In this way we may enjoy the benefits of using signatures while avoiding their main limitation.
	
	But this means that the signature transform is essentially operating as a layer within a neural network. It consumes a tensor of shape $(b, d, n)$ -- corresponding to a batch of size $b$ of paths in $\mathbb R^d$ that have been sampled $n$ times -- and returns a tensor of shape $(b, (d^{N + 1} - 1)/(d - 1))$, where $N$ is the number of terms used in the truncated signature.\footnote{As $(d^{N + 1} - 1)/(d - 1) = \sum_{k = 0}^N d^k$ is the number of scalar values in a signature with $N$ terms.} The signature is being used as a pooling operation.
	
	There is no reason to stop here. If the signature layer works well once then it is natural to seek to use it again. The obvious problem is that the signature transform consumes a stream of data and returns statistics which have no obvious stream-like qualities. The solution is to lift the input stream to a \emph{stream of streams}; for example, the stream of data $(x_1, \ldots, x_n)$ may be lifted to the `expanding windows' of $(\mathbf x_2, \ldots, \mathbf x_n)$, where $\mathbf x_i = (x_1, \ldots, x_i)$. Now apply the signature to each stream to obtain a stream of signatures $(\mathrm{Sig}^N(\mathbf x_2), \ldots, \mathrm{Sig}^N(\mathbf x_n))$, which is essentially a stream in Euclidean space.
	And now this new stream may be augmented via a neural network and the process repeated again, as many times as we wish.
	
	In this way the signature transform has been elevated from a one-time feature transformation to a first-class layer within a neural network. Thus we may reap the benefits of both the signature transform, with its strong corpus of mathematical theory, and the benefits of neural networks, with their great empirical success.
	
	Naturally all of this implies the need for an efficient implementation of the signature transform. Such concerns have motivated the creation of the spin-off Signatory project \cite{signatory}.
	
	The remainder of the paper is laid out as follows. In Section 2 we briefly discuss some related work, in Section 3 we detail the specifics of embedding the signature as a layer within a neural network. Sections 4 covers experiments; we demonstrate positive results for generative, supervised, and reinforcement learning problems. Section 5 is the conclusion. Appendix A provides an exposition of the theoretical properties of the signature, and Appendix B specifies implementation details.
	
	\section{Related Work}
	The signature transform is roughly analogous to the use of wavelets or Fourier transforms, and there are also related models based around these, for example \cite{Silvescu1999Fourier, Mingo2004Fourier, Gashler2016Fourier, Zhang1992wavelet}. We do not know of a detailed comparison between the use of these various transformations in the context of machine learning.
	
	Some related work using signatures has already been discussed in the previous section. We expand on their proposed models here.
	
	\begin{definition}
		Given a set $V$, the space of streams of data in $V$ is defined as
		\begin{equation*}
		\mathcal S(V) = \{ \mathbf x=(x_1, \ldots, x_n) : x_i \in V, n \in \mathbb N\}.
		\end{equation*}
		Given $\mathbf x=(x_1, \ldots, x_n) \in \mathcal S(V)$, the integer $n$ is called the length of $\mathbf x$.
	\end{definition}
	
		Two simple models utilising the signature layer are shown in Figure \ref{fig:simple-sig}.
	
	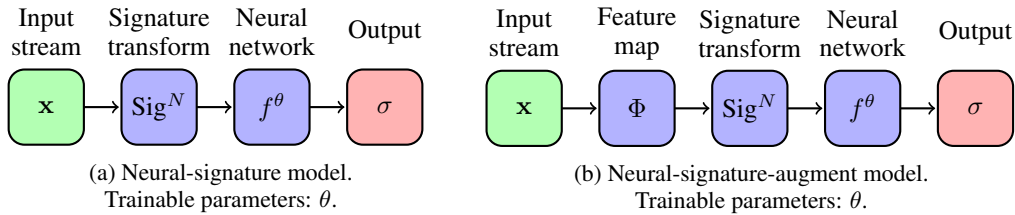
\begin{figure}[b]
		\captionsetup[subfigure]{justification=centering}
		\begin{subfigure}{0.45\textwidth}
			\centering
			\begin{tikzpicture}[shorten >=1pt,draw=black!50, node distance=\layersep]
			
			\draw[rounded corners=\corners, black, fill=green!30, thick]
			(0,0) rectangle ++(1,1);
			
%

			\node[text width=4em, text centered] at (0.5,0.5) {$\mathbf x$};
			
			\node[text width=4em, text centered] at (0.5,1.5) {Input stream};
			
			\draw[black, thick, ->] (1,0.5) -- (1.5,0.5);
			
			\draw[rounded corners=\corners, black, fill=blue!30, thick]
			(1.5,0) rectangle ++(1,1);
			
			\node[text width=4em, text centered] at (2,1.5) {Signature transform};
			\node[text width=4em, text centered] at (2,0.5) {  Sig$ ^N $  };

			\draw[black, thick, ->] (2.5,0.5) -- (3,0.5);
			
			\draw[rounded corners=\corners, black, fill=blue!30, thick]
			(3,0) rectangle ++(1,1);
			
			\node[text width=4em, text centered] at (3.5,1.5) {Neural network};
			\node[text width=6em, text centered] at (3.5,0.5) { $f^\theta$  };
			
			\draw[black, thick, ->] (4,0.5) -- (4.5,0.5);
			
			\draw[rounded corners=\corners, black, fill=red!30, thick]
			(4.5,0) rectangle ++(1,1);
			\node[text width=4em, text centered] at (5,1.5) {Output};
			\node[text width=6em, text centered] at (5,0.5) {$\sigma$};
			
			\end{tikzpicture}
			\caption{Neural-signature model. \\ Trainable parameters: $\theta$.}
			\label{fig:neural-sig}
		\end{subfigure}
		\begin{subfigure}{0.55\textwidth}
			\centering
			\begin{tikzpicture}[shorten >=1pt,draw=black!50, node distance=\layersep]
			
			\draw[rounded corners=\corners, black, fill=green!30, thick]
			(0,0) rectangle ++(1,1);
			
%
			
			\node[text width=4em, text centered] at (0.5,0.5) {$\mathbf x$};
			
			\node[text width=4em, text centered] at (0.5,1.5) {Input stream};

			\draw[black, thick, ->] (1,0.5) -- (1.5,0.5);
			
			\draw[rounded corners=\corners, black, fill=blue!30, thick]
			(1.5,0) rectangle ++(1,1);
			
			\node[text width=4em, text centered] at (2,1.5) {Feature map};
			\node[text width=4em, text centered] at (2,0.5) {  $\Phi$  };

			\draw[black, thick, ->] (2.5,0.5) -- (3,0.5);
			
			\draw[rounded corners=\corners, black, fill=blue!30, thick]
			(3,0) rectangle ++(1,1);
			
			\node[text width=4em, text centered] at (3.5,1.5) {Signature transform};
			\node[text width=4em, text centered] at (3.5,0.5) {  Sig$ ^N $  };

			\draw[black, thick, ->] (4,0.5) -- (4.5,0.5);
			
			\draw[rounded corners=\corners, black, fill=blue!30, thick]
			(4.5,0) rectangle ++(1,1);
			
			\node[text width=4em, text centered] at (5,1.5) {Neural network};
			\node[text width=6em, text centered] at (5,0.5) { $ f^{\theta}$  };
			
			\draw[black, thick, ->] (5.5,0.5) -- (6,0.5);
			\draw[rounded corners=\corners, black, fill=red!30, thick]
			(6,0) rectangle ++(1,1);
			
			\node[text width=4em, text centered] at (6.5,1.5) {Output};
			\node[text width=6em, text centered] at (6.5,0.5) {$\sigma$};
			
			\end{tikzpicture}
			\caption{Neural-signature-augment model. \\ Trainable parameters: $\theta$.}
			\label{fig:neural-sig-aug}
		\end{subfigure}
		
		\caption{Two simple architectures with a signature layer.}
		\label{fig:simple-sig}
	\end{figure}
	
	In principle the universal nonlinearity property of signatures (see Proposition \ref{prop:universal} in Appendix \ref{appendix:sigprop}) guarantees that the model shown in Figure \ref{fig:neural-sig}, is rich enough to learn any continuous function. (With the neural network taken to be a single linear layer and the input stream assumed to already be time-augmented.) In practice, of course, the signature must be truncated. Furthermore, it is not clear how to appropriately choose the truncation hyperparameter $N$. Thus a more practical approach is to remove the restriction that the neural network must be linear, and learn a nonlinear function instead. This approach has been applied successfully in various tasks \cite{lyonsfinance2014, yang2015chinese, xie2018learning, yang2016dropsample, yang2016deepwriterid, li2017lpsnet, yang2017leveraging, yang2016rotation}.
	
	An alternate model is shown in Figure \ref{fig:neural-sig-aug}. Following \cite{kiraly2016kernels, chevyrev2018signature}, a pointwise transformation could be applied to the stream before taking the signature transform. That is, applying the feature map $\Phi \colon \mathbb R^d \to \mathbb R^e$ to the $d$-dimensional stream of data $(x_1, \ldots, x_n) \in \mathcal S(\mathbb R^d)$ yields $(\Phi(x_1), \ldots, \Phi(x_n)) \in \mathcal S(\mathbb R^e)$; the signature of $\Phi(\mathbf x)$ may then potentially capture properties of the stream of data that will yield more effective models.

	\section{The signature transform as a layer in a neural network}
	However, there is not always a clear candidate for the feature map $\Phi$ and a good choice is likely to be data-dependent. Thus we propose to make $\Phi$ learnable by taking $\Phi = \Phi^\theta$ to be a neural network with trainable parameters $\theta$. In this case, we again obtain the neural network shown in Figure \ref{fig:neural-sig-aug}, except that $\Phi$ is now also learnable.
	
	The signature has now become a layer within a neural network. It consumes a tensor of shape $(b, d, n)$ -- corresponding to a batch of size $b$ of paths in $\mathbb R^d$ that have been sampled $n$ times -- and returns a tensor of shape $(b, (d^{N + 1} - 1)/(d - 1))$, where $N$ is the number of terms used in the truncated signature.
	
Despite being formed of integrals, the signature is in fact straightforward and efficient to compute exactly, see Section \ref{subsection:computation} in Appendix \ref{appendix:sigprop}. More than that, the computation may in fact be described in terms of standard tensor operations. As such it may be backpropagated through without difficulty. 
	
	
	\subsection{Stream-preserving neural networks}
	
	Let $\mathbf x = (x_1, \ldots, x_n) \in \mathcal S(\mathbb R^d)$. Whatever the choice of $\Phi^\theta$, it must preserve the stream-like nature of the data if we are to take a signature afterwards. The simplest way of doing this is to have $\Phi^\theta$ map $\mathbb R^d \to \mathbb R^e$, so that it operates pointwise. This defines $\Phi(\mathbf x)$ by
	\begin{equation}\label{eq:pointwise}
	\Phi(\mathbf x) = (\Phi^\theta(x_1), \ldots, \Phi^\theta(x_n)) \in \mathcal S(\mathbb R^e).
	\end{equation}
	Another way to preserve the stream-like nature is to sweep a one dimensional convolution along the stream; more generally one could sweep a whole feedforward network along the stream. For some $ m \in \mathbb N $ and $\Phi^\theta \colon \mathbb{R}^{d \times m} \to \mathbb R^e$ this defines $\Phi(\mathbf x)$ by
	\begin{equation}\label{eq:convolution}
	\Phi(\mathbf x) = (\Phi^\theta(x_1, \ldots, x_m), \ldots, \Phi^\theta(x_{n - m + 1}, \ldots, x_n)) \in \mathcal S(\mathbb R^e).
	\end{equation}
	More generally still the network could be recurrent, by having memory. Let $\Phi_{0} = 0$, fix $m \in \mathbb N$, and define $\Phi_k = \Phi^\theta\left(x_k, \ldots, x_{k + m}; \Phi_{k - 1}\right)$ for $k=1, \ldots, n - m + 1$. Then define $\Phi(\mathbf x)$ by
	\begin{equation}\label{eq:recurrent}
	\Phi(\mathbf x) = (\Phi_1, \ldots, \Phi_{n -m +1}) \in \mathcal S(\mathbb R^e).
	\end{equation}
	
	\subsection{Stream-like data}
	
	It is worth taking a moment to think what is really meant by `stream-like nature'. The signature transform is defined on paths; it is applied to a stream of data in $\mathcal S(\mathbb R^d)$ by first interpolating the data into a path and then taking the signature.
	
	The data is treated as a discretisation or set of observations of some underlying path. Note that there is nothing wrong with the path itself having a discrete structure to it; for example a sentence.
	
	In principle one could reshape a tensor of shape $(b, nd)$ with no stream-like nature into one of shape $(b, d, n)$, and then take the signature. However it is not clear what this means mathematically. There is no underlying path. The signature is at this point an essentially arbitrary transformation, without the mathematical guarantees normally associated with it.
	
	\subsection{Stream-preserving signatures, using lifts}
	
	We would like to apply the signature layer multiple times. However applying the signature transform consumes the stream-like nature of the data, which prevents this. The solution is to construct a stream of signatures in the following way: given a stream $\mathbf x = (x_1, \ldots, x_n) \in \mathcal S(\mathbb R^d)$, let $\mathbf x_k = (x_1, \ldots, x_k)$ for $k=2, \ldots, n$, and apply the signature to each $\mathbf x_k$ to obtain the stream
	\begin{equation}\label{eq:firstlift}
	(\mathrm{Sig}^N(\mathbf x_2), \ldots, \mathrm{Sig}^N(\mathbf x_n)) \in \mathcal S(\mathbb R^{(d^{N + 1} - 1)/(d - 1)}).
	\end{equation}
	The shortest stream it is meaningful to take the signature of is of length two, which is why there is no corresponding $\mathrm{Sig}^N(\mathbf x_1)$ term.
	
	In this way the stream-like nature of the data is preserved through the signature transform.
	
	This notion may be generalised: let
	\begin{equation*}
	\ell = (\ell^1, \ell^2, \ldots, \ell^v) \colon \mathcal S(\mathbb R^d) \to \mathcal S(\mathcal S(\mathbb R^e)),
	\end{equation*}
	which we refer to as a \emph{lift} into the space of streams of streams (and $v$ will likely depend on the length of the input to $\ell$). Then apply the signature stream-wise to define $\mathrm{Sig}^N(\ell(\mathbf x))$ by
	\begin{equation}\label{eq:sigdef}
	\mathrm{Sig}^N(\ell(\mathbf x)) = \left(\mathrm{Sig}^N(\ell^1(\mathbf x)), \ldots, \mathrm{Sig}^N(\ell^v(\mathbf x))\right) \in \mathcal S(\mathbb R^{(e^{N + 1} - 1)/(e - 1)}).
	\end{equation}
	
	In the example of equation \eqref{eq:firstlift}, $\ell$ is given by
	\begin{equation}\label{eq:expanding}
	\ell(\mathbf x) = (\mathbf x_2, \ldots, \mathbf x_n).
	\end{equation}
	
	Other plausible choices for $\ell$ are to cut up $\mathbf x$ into multiple pieces, for example
	\begin{equation}\label{eq:block}
	\ell(\mathbf x) = ((x_1, x_2), (x_3, x_4), \ldots, (x_{2\lfloor n/2 \rfloor - 1}, x_{2\lfloor n/2 \rfloor})),
	\end{equation}
	or to take a sliding window
	\begin{equation}\label{eq:sliding}
	\ell(\mathbf x) = ((x_1, x_2, x_3), (x_2, x_3, x_4), \ldots, (x_{n - 2}, x_{n - 1}, x_n)).
	\end{equation}
	
\subsection{Multiple signature layers}
	
	By inserting lifts, the signature transform may be composed as many times as desired. That is, suppose we wish to learn a map from $\mathcal S(\mathbb R^d)$ to $\mathcal X$, where $\mathcal X$ is some set. (Which may be finite for a classification problem or infinite for a regression problem.) Let $c_i, d_i, e_i, N_i \in \mathbb N$ be such that $d_1 = d$ and $d_{i + 1} = (c_{i}^{N_i + 1} - 1) / (c_{i} - 1)$, for $i = 1, \ldots, k$.
	
	Let
	\begin{align*}
	\Phi_i^{\theta_i} \colon \mathbb{R}^{d_i \times m_i} \to \mathbb R^{e_i},\qquad
	\ell_i \colon \mathcal S(\mathbb R^{e_i}) \to \mathcal S(\mathcal S(\mathbb R^{c_{i}})),\qquad
	f^{\theta_{k+1}} \colon \mathcal S(\mathbb R^{(c_{k}^{N_k + 1} - 1) / (c_{k} - 1)}) \to \mathcal X,
	\end{align*}
	where $\Phi_i^{\theta_i}$ and $\ell_i$ are defined in the manner of equations \eqref{eq:pointwise}--\eqref{eq:recurrent} and \eqref{eq:expanding}--\eqref{eq:sliding}, and $\theta_1, \ldots, \theta_{k+1}$ are some trainable parameters. Then defining compositions in the manner of equations \eqref{eq:pointwise}--\eqref{eq:sigdef}, let
	\begin{equation*}
	\sigma = \left(f^{\theta_{k + 1}} \circ \mathrm{Sig}^{N_k} \circ \ell_k \circ \Phi_k^{\theta_k} \circ \cdots \circ \Phi_2^{\theta_2} \circ \mathrm{Sig}^{N_1} \circ \ell_1 \circ \Phi_1^{\theta_1}\right)(\mathbf x).
	\end{equation*}
	
	\begin{figure}[b]
		\begin{tikzpicture}[shorten >=1pt,draw=black!50, node distance=\layersep]
		
		\draw[rounded corners=\corners, black, fill=green!30, thick]
		(1,0) rectangle ++(1,1);
		\node[text width=4em, text centered] at (1.5,0.5) {$\mathbf x$};
		\node[text width=4em, text centered] at (1.5,1.5) {Input stream};
		
		\draw[black, thick, ->] (2,0.5) -- (2.5,0.5);
		
		\draw[rounded corners=\corners, black, fill=blue!30, thick]
		(2.5,0) rectangle ++(1,1);		
		\node[text width=4em, text centered] at (3,1.5) {Neural network};
		\node[text width=4em, text centered] at (3,0.5) {  $ \Phi_1^{\theta_1} $  };
		
		\draw[black, thick, ->] (3.5,0.5) -- (4,0.5);
		
		\draw[rounded corners=\corners, black, fill=blue!30, thick]
		(4,0) rectangle ++(1,1);
		\node[text width=4em, text centered] at (4.5,1.5) {Lift};
		\node[text width=4em, text centered] at (4.5,0.5) {  $\ell_1$  };
		
		\draw[black, thick, ->] (5,0.5) -- (5.5,0.5);
		
		\draw[rounded corners=\corners, black, fill=blue!30, thick]
		(5.5,0) rectangle ++(1,1);
		\node[text width=4em, text centered] at (6,1.5) {Signature transform};
		\node[text width=4em, text centered] at (6,0.5) {  Sig$ ^{N_1} $  };
		
		\draw[black, thick, ->] (6.5,0.5) -- (7,0.5);
		
		\draw[rounded corners=\corners, black, fill=blue!30, thick]
		(7,0) rectangle ++(1,1);
		\node[text width=4em, text centered] at (7.5,1.5) {Neural network};
		\node[text width=4em, text centered] at (7.5,0.5) {$\Phi_2^{\theta_2}$};
		
		\draw[black, thick, ->] (8,0.5) -- (8.25,0.5);
		
		\node[text width=4em, text centered] at (8.5,0.5) {$\ldots$};	
		
		\draw[black, thick, ->] (8.75,0.5) -- (9,0.5);
		
		\draw[rounded corners=\corners, black, fill=blue!30, thick]
		(9,0) rectangle ++(1,1);
		\node[text width=4em, text centered] at (9.5,1.5) {Lift};
		\node[text width=4em, text centered] at (9.5,0.5) {$\ell_k$};
		
		\draw[black, thick, ->] (10,0.5) -- (10.5,0.5);
		
		\draw[rounded corners=\corners, black, fill=blue!30, thick]
		(10.5,0) rectangle ++(1,1);
		\node[text width=4em, text centered] at (11,1.5) {Signature transform};
		\node[text width=4em, text centered] at (11,0.5) {$\mathrm{Sig}^{N_k}$};
		
		\draw[black, thick, ->] (11.5,0.5) -- (12,0.5);
		
		\draw[rounded corners=\corners, black, fill=blue!30, thick]
		(12,0) rectangle ++(1,1);
		\node[text width=4em, text centered] at (12.5,1.5) {Neural network};
		\node[text width=4em, text centered] at (12.5,0.5) {$f^{\theta_{k+1}}$};
		
		\draw[black, thick, ->] (13,0.5) -- (13.5,0.5);
		
		\draw[rounded corners=\corners, black, fill=red!30, thick]
		(13.5,0) rectangle ++(1,1);
		\node[text width=4em, text centered] at (14,1.5) {Output};
		\node[text width=4em, text centered] at (14,0.5) {$\sigma$};		
		\end{tikzpicture}
		\centering
		\caption{Deep signature model. Trainable parameters: $\theta_1, \ldots, \theta_{k + 1}$.}
		\label{fig:deep signatures}
	\end{figure}
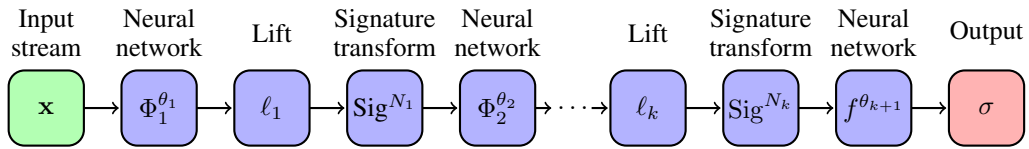
	
	This defines the \emph{deep signature model}, summarised in Figure \ref{fig:deep signatures}.
	
	An important special case is when $V = \mathcal S(\mathbb R^e)$, so that the final network $f^{\theta_{k+1}}$ is stream-preserving. Then the overall model $\mathbf x \mapsto \sigma$ is also stream-preserving. See for example Section \ref{section:generative}.
	
	Note that in principle it is acceptable to take the trivial lift to a sequence of a single element,
	\begin{equation}\label{eq:nolift}
	\ell(\mathbf x) = (\mathbf x).
	\end{equation}
	Taking the signature of this will then essentially remove the stream-like nature, however, so it is suitable only for the final lift of a deep signature model. We observe in particular that this is what is done in the models described in Figure \ref{fig:simple-sig}, which we identify as special cases of the deep signature model, lacking also any learned transformation before the signature.
	
	It is easy to see that the deep signature model exhibits the universal approximation property. This fact follows from the universal approximation theorem for neural networks \cite{pinkus} and from the universal nonlinearity property of signatures (see Proposition \ref{prop:universal} in Appendix \ref{appendix:sigprop}).
	
	
	\subsection{Implementation}
When using the signature transform as a feature transformation, then it suffices to just pre-process and save the entire dataset before training. However when the signature transform is placed within a neural network then the signature transform must be evaluated and backpropagated through for each step of training; this is much more computationally intensive. This has motivated the creation of the separate spin-off Signatory project \cite{signatory}, to efficiently perform and backpropagate through the signature transform.
	
\subsection{Inverting the truncated signature}

How well does a truncated signature encode the original stream of data? A simple experiment is to attempt to recover the original stream of data given its truncated signature. We remark that finding a mathematical description of this inversion is a challenging task \cite{lyons2018inverting, chang2017signature, chang2018effective}.

Fix a stream of data $\mathbf x = (x_1, \ldots, x_n) \in \mathcal S(\mathbb R^d)$. Assume that the truncated signature $\mathrm{Sig}^N(\mathbf x)$ and the number of steps $n\in \mathbb N$ are known. Now apply gradient descent to minimise
\begin{equation*}
L(\mathbf y; \mathbf x) = \left \lVert \mathrm{Sig}^N(\mathbf y) - \mathrm{Sig}^N(\mathbf x)\right \rVert_2^2\quad \mbox{for }\mathbf y=(y_1, \ldots, y_n)\in \mathcal S(\mathbb R^d).
\end{equation*}

Figure \ref{fig:invertedsig} shows four handwritten digits from the PenDigits dataset \cite{UCI}. The solid blue path is the original path $\mathbf x$, whilst the dashed orange path is the reconstructed path $\mathbf y$ minimising $L(\mathbf y; \mathbf x)$. Truncated signatures of order $N=12$ were used for this task. We see that the truncated signatures have managed to encode the input paths $\mathbf x$ almost perfectly.

\begin{figure}[h]
\centering
\begin{minipage}{.2\textwidth}
  \centering
  \includegraphics[width=0.9\linewidth]{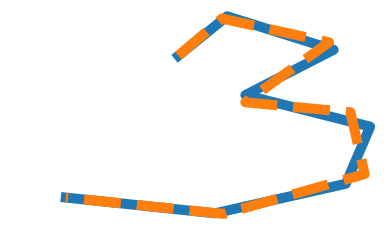}
\end{minipage}
\begin{minipage}{.2\textwidth}
  \centering
  \includegraphics[width=0.9\linewidth]{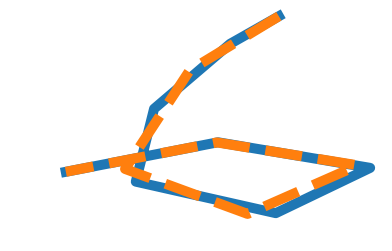}
\end{minipage}
\begin{minipage}{.2\textwidth}
  \centering
  \includegraphics[width=0.9\linewidth]{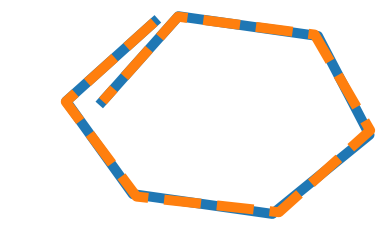}
\end{minipage}
\begin{minipage}{.2\textwidth}
  \centering
  \includegraphics[width=0.9\linewidth]{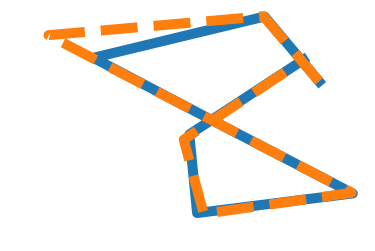}
\end{minipage}%
\caption{Original path (blue) and path reconstructed from its signature (dashed orange) for four handwritten digits in the PenDigits dataset \cite{UCI}.}
\label{fig:invertedsig}
\end{figure}

\section{Numerical experiments}
\subsection{A generative model for a stochastic process}\label{section:generative}

Generative models are typically trained to learn to transform random noise to a target distribution. One common approach are Generative Adversarial Networks \cite{goodfellow2014generative}. An alternative approach is to define a distance on the space of distributions by embedding them into a Reproducing Kernel Hilbert Space. The discriminator is then a fixed two-sample test based on a kernel maximum mean discrepancy. This is known as a Generative Moment Matching Network \cite{gmmn2015, traininggmmns2015, gretton2017}.

With this framework we propose a deep signature model to generate sequential data. The discriminator is as in \cite{kiraly2016kernels, chevyrev2018signature}. The natural choice for random noise is Brownian motion $B_t$.

Define the kernel $k \colon \mathcal S(\mathbb R^d) \times \mathcal S(\mathbb R^d) \to \mathbb R$ by
\begin{align*}
k(\mathbf{x},\mathbf{y}) &= \left (\mathrm{Sig}^N \left (\lambda_{\mathbf{x}} \mathbf x\right ), \mathrm{Sig}^N\left (\lambda_{\mathbf{y}}\mathbf y\right ) \right ),
\end{align*}
where $\lambda_{\mathbf x}\in \mathbb R$ is a certain normalising constant which guarantees that $k$ is the kernel of a Reproducing Kernel Hilbert Space, and $(\,\cdot\, , \,\cdot\,)$ denotes the dot product.

Given $n$ samples $\{\mathbf x^{(i)}\}_{i=1}^n \subseteq \mathcal S(\mathbb R^d)$ from the generator and $m$ samples $\{\mathbf y^{(i)}\}_{i=1}^m \subseteq \mathcal S(\mathbb R^d)$ from the target distribution, define the loss $T$ by
\begin{equation*}
T\left (\{\mathbf x^{(i)}\}_{i=1}^n, \{\mathbf y^{(i)}\}_{i=1}^m\right ) = \dfrac{1}{n^2}\sum_{i, j} k(\mathbf{x}^{(i)},\mathbf{x}^{(j)}) - \dfrac{2}{nm}\sum_{i, j} k(\mathbf{x}^{(i)},\mathbf{y}^{(j)}) + \dfrac{1}{m^2}\sum_{i, j} k(\mathbf{y}^{(i)},\mathbf{y}^{(j)}).
\end{equation*}

\newcommand{\offset}{-0.5}  
\newcommand{\otheroffset}{-0.2}  
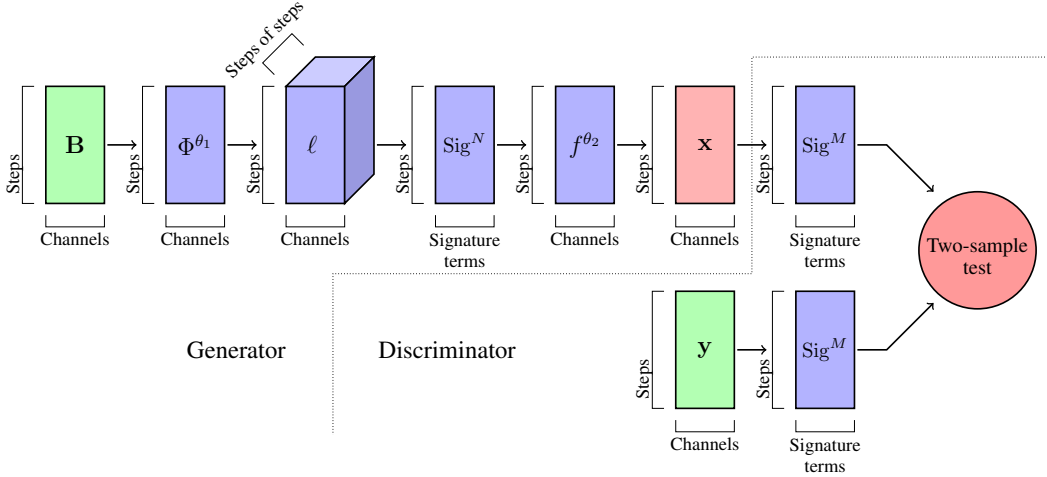
\begin{figure}[t]
	\centering
	\resizebox{\textwidth}{!}{%
		\begin{tikzpicture}
		\node (Position) at (0,0) {};
			\draw[black, fill=green!30, thick] (Position) rectangle ++(1,2);	
			\draw ($ (Position) + (-.2,0) $) -- ++(-0.2,0) -- ++(0,2) -- ++(0.2,0); 
			\draw ($ (Position) + (0,-0.2) $) -- ++(0,-.2) -- ++(1,0) -- ++(0,.2); 
			\node[align=center,font=\large](centre) at ($ (Position) + (.5,1) $){$\mathbf{B} $}; 
			\node[align=center,font=\small, rotate=90](left) at ($ (Position) + (-.53,0.5) $){Steps}; 
			\node[align=center,font=\small](bottom) at ($ (Position) + (.5,-0.6) $){Channels}; 
			\draw[black, thick, ->] ($ (Position) + (1.25 + \otheroffset,1) $) -- ++(1 + \offset,0);
		
		\node (Position) at (2.75 + \offset + \otheroffset,0) {};
			\draw[black, fill=blue!30, thick] (Position) rectangle ++(1,2);	
			\draw ($ (Position) + (-.2,0) $) -- ++(-0.2,0) -- ++(0,2) -- ++(0.2,0); 
			\draw ($ (Position) + (0,-0.2) $) -- ++(0,-.2) -- ++(1,0) -- ++(0,.2); 
			\node[align=center,font=\large](centre) at ($ (Position) + (.5,1) $){$ \Phi^{\theta_1} $}; 
			\node[align=center,font=\small, rotate=90](left) at ($ (Position) + (-.53,0.5) $){Steps}; 
			\node[align=center,font=\small](bottom) at ($ (Position) + (.5,-0.6) $){Channels}; 
			\draw[black, thick, ->] ($ (Position) + (1.25 + \otheroffset,1) $) -- ++(1 + \offset,0);
		
		\node (Position) at (5.5 + 2 * \offset + 2 * \otheroffset,0) {};
			\draw[black, fill=blue!30, thick] (Position) rectangle ++(1,2);	
			\draw[fill=blue!20, thick]  ($ (Position) + (0,2) $) -- ++(.5,.5) -- ++(1,0) -- ++(-.5,-.5) -- cycle; 
			\draw[fill=blue!70!black!40, thick]  ($ (Position) + (1,0) $) -- ++(.5,.5) -- ++(0,2) -- ++(-.5,-.5) -- cycle; 
			\draw ($ (Position) + (-.2,0) $) -- ++(-0.2,0) -- ++(0,2) -- ++(0.2,0); 
			\draw ($ (Position) + (0,-0.2) $) -- ++(0,-.2) -- ++(1,0) -- ++(0,.2); 
			\draw ($ (Position) + (-.2,2.2) $) -- ++(-.2,.2) -- ++(.5,.5) -- ++(.2,-.2); 
			\node[align=center, font=\large](centre) at ($ (Position) + (.5,1) $){ $\ell$ }; 
			\node[align=center,font=\small, rotate=90](left) at ($ (Position) + (-.53,0.5) $){Steps}; 
			\node[align=center,font=\small](bottom) at ($ (Position) + (.5,-0.6) $){Channels}; 
			\node[align=center,font=\small, rotate=45](top) at ($ (Position) + (-.35,2.8) $){Steps of steps}; 
			\draw[black, thick, ->] ($ (Position) + (1.75 + \otheroffset,1) $) -- ++(1 + \offset,0);
		
		\node (Position) at (8.75 + 3 * \offset + 3 * \otheroffset ,0) {};
			\draw[black, fill=blue!30, thick] (Position) rectangle ++(1,2);	
			\draw ($ (Position) + (-.2,0) $) -- ++(-0.2,0) -- ++(0,2) -- ++(0.2,0); 
			\draw ($ (Position) + (0,-0.2) $) -- ++(0,-.2) -- ++(1,0) -- ++(0,.2); 
			\node[align=center](centre) at ($ (Position) + (.5,1) $){$ \mathrm{Sig}^N $}; 
			\node[align=center,font=\small, rotate=90](left) at ($ (Position) + (-.53,0.5) $){Steps}; 
			\node[align=center,font=\small](bottom) at ($ (Position) + (.5,-0.8) $){Signature \\ terms}; 
			\draw[black, thick, ->] ($ (Position) + (1.25 + \otheroffset,1) $) -- ++(1 + \offset,0);
		
		\node (Position) at (11.5 + 4 * \offset + 4 * \otheroffset,0) {};
			\draw[black, fill=blue!30, thick] (Position) rectangle ++(1,2);	
			\draw ($ (Position) + (-.2,0) $) -- ++(-0.2,0) -- ++(0,2) -- ++(0.2,0); 
			\draw ($ (Position) + (0,-0.2) $) -- ++(0,-.2) -- ++(1,0) -- ++(0,.2); 
			\node[align=center,font=\large](centre) at ($ (Position) + (.5,1) $){$ f^{\theta_2} $}; 
			\node[align=center,font=\small, rotate=90](left) at ($ (Position) + (-.53,0.5) $){Steps}; 
			\node[align=center,font=\small](bottom) at ($ (Position) + (.5,-0.6) $){Channels}; 
			\draw[black, thick, ->] ($ (Position) + (1.25 + \otheroffset,1) $) -- ++(1 + \offset,0);
			
		\node (Position) at (14.25 + 5 * \offset + 5 * \otheroffset,0) {};
			\draw[black, fill=red!30, thick] (Position) rectangle ++(1,2);	
			\draw ($ (Position) + (-.2,0) $) -- ++(-0.2,0) -- ++(0,2) -- ++(0.2,0); 
			\draw ($ (Position) + (0,-0.2) $) -- ++(0,-.2) -- ++(1,0) -- ++(0,.2); 
			\node[align=center,font=\large](centre) at ($ (Position) + (.5,1) $){$ \mathbf x $}; 
			\node[align=center,font=\small, rotate=90](left) at ($ (Position) + (-.53,0.5) $){Steps}; 
			\node[align=center,font=\small](bottom) at ($ (Position) + (.5,-0.6) $){Channels}; 
			\draw[black, thick, ->] ($ (Position) + (1.25 + \otheroffset,1) $) -- ++(1 + \offset,0);

		\node (Position) at (17 + 6 * \offset + 6 * \otheroffset, 0) {};
			\draw[black, fill=blue!30, thick] (Position) rectangle ++(1,2);	
			\draw ($ (Position) + (-.2,0) $) -- ++(-0.2,0) -- ++(0,2) -- ++(0.2,0); 
			\draw ($ (Position) + (0,-0.2) $) -- ++(0,-.2) -- ++(1,0) -- ++(0,.2); 
			\node[align=center](centre) at ($ (Position) + (.5,1) $){$ \mathrm{Sig}^M $}; 
			\node[align=center,font=\small, rotate=90](left) at ($ (Position) + (-.53,0.5) $){Steps}; 
			\node[align=center,font=\small](bottom) at ($ (Position) + (.5,-0.8) $){Signature \\ terms}; 
			\draw[black, thick, ->] ($ (Position) + (1.25 + \otheroffset,1) $) -- ++(1 + \offset,0) -- ++(.8,-.8);

			\draw[densely dotted] ($ (Position) + (4.25,2.5) $) -- ++(-5,0) -- ++(0,-3.7) -- ++(-7.15,0) -- ++(0,-2.75);
			\node[align=center, font=\large](left) at ($ (Position) + (-7.75 - 1.8,-2.5) $){Generator}; 
			\node[align=center, font=\large](left) at ($ (Position) + (-7.75 + 1.8,-2.5) $){Discriminator};
			\draw[fill=red!40, thick] ($ (Position) + (3.6 + \offset,-0.8) $) circle (1);
			\node[align=center](centre) at ($ (Position) + (3.6 + \offset,-0.9) $){Two-sample \\ test}; 			
		
		\node (Position) at (14.25 + 5 * \offset + 5 * \otheroffset,-3.5) {};
			\draw[black, fill=green!30, thick] (Position) rectangle ++(1,2);	
			\draw ($ (Position) + (-.2,0) $) -- ++(-0.2,0) -- ++(0,2) -- ++(0.2,0); 
			\draw ($ (Position) + (0,-0.2) $) -- ++(0,-.2) -- ++(1,0) -- ++(0,.2); 
			\node[align=center,font=\large](centre) at ($ (Position) + (.5,1) $){$ \mathbf {y} $}; 
			\node[align=center,font=\small, rotate=90](left) at ($ (Position) + (-.53,0.5) $){Steps}; 
			\node[align=center,font=\small](bottom) at ($ (Position) + (.5,-0.6) $){Channels}; 
			\draw[black, thick, ->] ($ (Position) + (1.25 + \otheroffset,1) $) -- ++(1 + \offset,0);
		
		\node (Position) at (17 + 6 * \offset + 6 * \otheroffset,-3.5) {};
			\draw[black, fill=blue!30, thick] (Position) rectangle ++(1,2);	
			\draw ($ (Position) + (-.2,0) $) -- ++(-0.2,0) -- ++(0,2) -- ++(0.2,0); 
			\draw ($ (Position) + (0,-0.2) $) -- ++(0,-.2) -- ++(1,0) -- ++(0,.2); 
			\node[align=center](centre) at ($ (Position) + (.5,1) $){$ \mathrm{Sig}^M $}; 
			\node[align=center,font=\small, rotate=90](left) at ($ (Position) + (-.53,0.5) $){Steps}; 
			\node[align=center,font=\small](bottom) at ($ (Position) + (.5,-0.8) $){Signature \\ terms}; 
			\draw[black, thick, ->] ($ (Position) + (1.25 + \otheroffset,1) $) -- ++(1 + \offset,0) -- ++(.8,.8);

		\end{tikzpicture}
	}
	\caption{Generative model architecture. Trainable parameters: $\theta_1, \theta_2$. There is an implicit batch dimension throughout.}
	\label{GAN-arch}	
\end{figure}

Let the input to the network be time-augmented Brownian motion
\begin{equation*}
\mathbf{B} = ((t_1, B_{t_1}), \ldots, (t_n, B_{t_n})) \in \mathcal S(\mathbb R^2).
\end{equation*}
Given two stream-preserving neural networks $\Phi^{\theta_1}$ and $f^{\theta_2}$, and a lift $\ell$, then the generative model is defined by
\begin{equation*}
\mathbf x = (f^{\theta_2} \circ \mathrm{Sig}^N \circ \ell \circ \Phi^{\theta_1})(\mathbf{B}).
\end{equation*}

The overall model is shown in Figure \ref{GAN-arch}.  In a nice twist, both the generator and the discriminator involve the signature.

\newlength{\backupintextsep}
\setlength{\backupintextsep}{\intextsep}
\setlength{\intextsep}{-4pt}
\begin{wrapfigure}[14]{r}{0pt}
\includegraphics[width=0.5\textwidth]{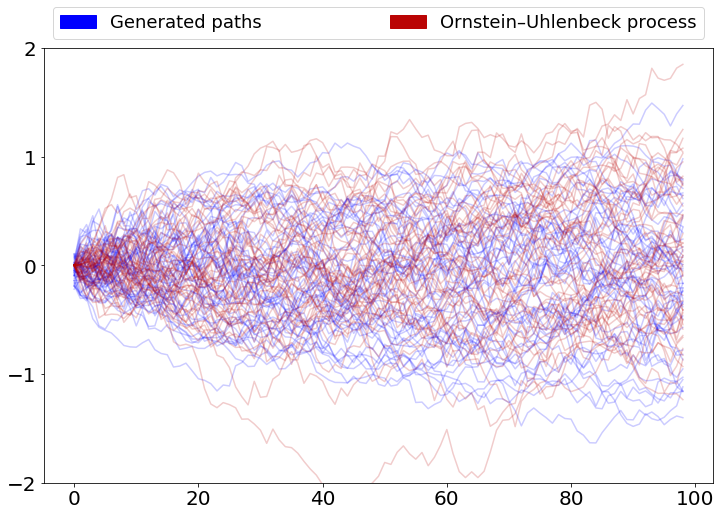}
\captionsetup{width=0.9\linewidth}
\caption{Generated paths alongside the original paths.}
\label{fig:OU}
\end{wrapfigure}

\vspace{-1pt}Observe \setlength{\intextsep}{\backupintextsep} how the generative part is a particular case of the deep signature model, and that furthermore the whole generator-discriminator pair is also a particular case of the deep signature model, with the trivial lift of equation \eqref{eq:nolift} before the second signature layer.

We applied the proposed model to a dataset of $1024$ realisations of an Ornstein--Uhlenbeck process \cite{uhlenbeck1930theory}. The loss was minimised at $6.6 \times 10^{-4}$, which implies that the generated paths are statistically almost indistinguishable from the real Ornstein--Uhlenbeck process. Figure \ref{fig:OU} shows the generated paths alongside the original ones. Further implementation details are in Appendix \ref{appendix:implementation}.
	
\subsection{Supervised learning with fractional Brownian motion}
Fractional Brownian motion \cite{mishura2008stochastic} is a Gaussian process $B^H \colon [0,\infty) \rightarrow \mathbb{R}$ that generalises Brownian motion. It is self-similar and exhibits fractal-like behaviour. 
Fractional Brownian motion depends upon a parameter $H \in (0, 1)$, known as the Hurst parameter. Lower Hurst parameters result in noticeably rougher paths. The case of $H = 1/2$ corresponds to usual Brownian motion. 
Fractional Brownian motion has been successfully used to model phenomena in diverse fields. For example, empirical evidence from financial markets \cite{gatheral2014volatility} suggests that log-volatility is well modelled by fractional Brownian motion with Hurst parameter $H \approx 0.1$.


Estimating the Hurst parameter of a fractional Brownian motion path is considered a nontrivial task because of the paths' non-stationarity and long range dependencies \cite{lacasa2009visibility}. We train a variety of models to perform this estimation. That is, to learn the map $\mathbf x^H \mapsto H$,
where
\begin{equation*}
\mathbf x^H = ((t_0, B^H_{t_0}), \ldots, (t_n, B^H_{t_n})) \in \mathcal S(\mathbb R^2)
\end{equation*}
for some realisation of $B^H$.

\setlength{\intextsep}{0pt}
\begin{wrapfigure}[20]{r}{0pt}
\begin{minipage}{0.5\textwidth}
	\centering	
	\includegraphics[width=\textwidth]{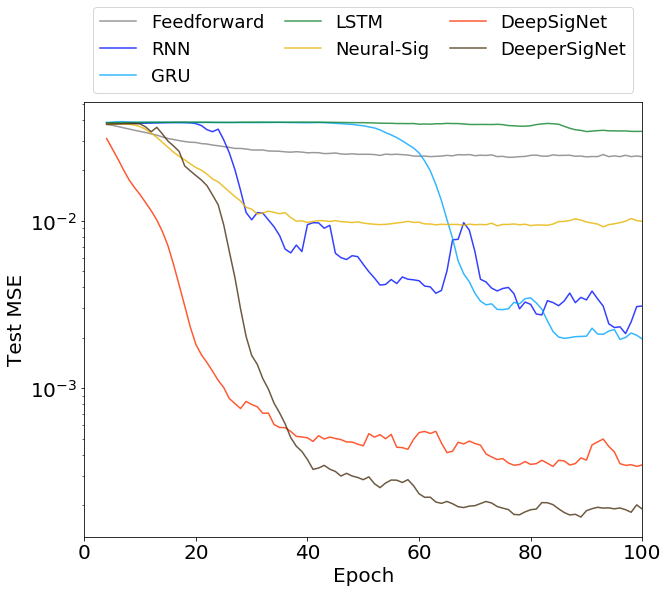}
	\captionsetup{width=0.9\linewidth}
	\caption{Performance at estimating the Hurst parameter for various models, with and without signatures, for a particular (typical) training run.}
	\label{fig:fbm}
\end{minipage}
\end{wrapfigure}

The \setlength{\intextsep}{\backupintextsep} results are shown in Figure \ref{fig:fbm} and Table \ref{table:fbm results}. Also shown in Table \ref{table:fbm results} are the results of the rescaled range method \cite{Hurst1951}, which is a mathematically derived method rather than a learned method.

RNN, GRU and LSTM models provide baselines in the context of recurrent neural networks. 
The simple Neural-Sig model outlined previously in Figure \ref{fig:neural-sig} provides a baseline from the context of signatures.

DeepSigNet and DeeperSigNet are both deep signature models of the form given by Figure \ref{fig:deep signatures}. DeepSigNet has a single large Neural-Lift-Signature block, whilst DeeperSigNet has three smaller ones.

We observe that traditional signature based models perform slightly worse than traditional recurrent models, but that deep signature models outperform all other models by at least an order of magnitude. Further implementation details are found in Appendix \ref{appendix:implementation}.
\vspace{0.5em}  


\begin{table}[t]
	\centering
	\vspace{4em}
	\caption{Final test mean squared error (MSE) for the different models, averaged over 3 training runs, ordered from largest to smallest.}
	\label{table:fbm results}
	\begin{tabular}{l c c c}
		\toprule
		{} & \multicolumn{2}{c}{Test MSE}\\
		\cmidrule(r){2-3}
		{} & {Mean} & {Variance} & {\# Params} \\
		\cmidrule(r){2-4}
		$\text{Rescaled Range}$ & $7.2 \times 10^{-2}$ &$3.7 \times 10^{-3}$ & N/A \\
		$\text{LSTM}$  & $4.3 \times 10^{-2}$ &$8.0 \times 10^{-3}$ & 12961  \\
		$\text{Feedforward}$ & $2.8 \times 10^{-2}$ & $3.0 \times 10^{-3}$ & 10209 \\
		$\text{Neural-Sig}$  & $1.1 \times 10^{-2}$ &$8.2 \times 10^{-4}$ & 10097 \\
		$\text{GRU}$  & $3.3 \times 10^{-3}$ &$1.3 \times 10^{-3}$ & 9729   \\
		$\text{RNN}$  & $1.7 \times 10^{-3}$ &$4.9 \times 10^{-4}$ & 10091  \\	
		$\text{DeepSigNet}$  & $2.1 \times 10^{-4}$ &$8.7 \times 10^{-5}$ & 9261  \\
		$\text{DeeperSigNet}$  & $1.6 \times 10^{-4}$ &$2.1 \times 10^{-5}$ & 9686  \\
		\bottomrule
	\end{tabular}
\end{table}

\subsection{Non-Markovian deep reinforcement learning}\label{section:rl}

Finally we show how these ideas may be extended, by demonstrating a model that adds a residual connection to the deep signature model; it may also be interpreted as using signatures as the memory of a recurrent neural network.

As an example, we apply this architecture to tackle a non-Markovian reinforcement learning problem. This means that the optimal action depends not just on the current state of the environment, but upon the history of past states, so that the agent must maintain a memory.

Let  $\Phi^{\theta_1} \colon \mathbb R^d \to \mathbb R^e$ and $f^{\theta_2} \colon \mathbb R^{d + (e^{N + 1} - 1)/(e - 1)} \to \{\mathrm{actions}\}$ be functions depending on learnable parameters $\theta_1$, $\theta_2$. Given input $x_i \in \mathbb R^d$ at time $i$, let
\begin{equation*}
y_i = \Phi^{\theta_1}(x_i),\qquad\sigma_i = \sigma_{i - 1} \otimes \mathrm{Sig}^N((y_{i - 1}, y_{i})),\qquad a_i = f^{\theta_2}(x_i, \sigma_i),
\end{equation*}
where $a_i$ is the action proposed by the network at time $i$, and $y_i$ and $\sigma_i$ are the memory at time $i$, and $\otimes$ denotes the tensor product as in \ref{def:tensor-product} in Appendix \ref{appendix:sigprop}.

The model is summarised in Figure \ref{fig:rl-model-rnn} as a recurrent neural network with signature-based memory. Note that $y_i$ is preserved in memory only to compute the signature at the next time step, as the shortest path it is meaningful to compute the signature of is of length two.

However, note that by Proposition \ref{prop:chenefficient} in Appendix \ref{appendix:sigprop},
\begin{equation*}
\sigma_i = \mathrm{Sig}^N (\Phi^{\theta_1}(x_1), \ldots, \Phi^{\theta_1}(x_i)) \in \mathbb R^{(e^{N + 1} - 1)/(e - 1)}.
\end{equation*}

\setlength{\intextsep}{0pt}
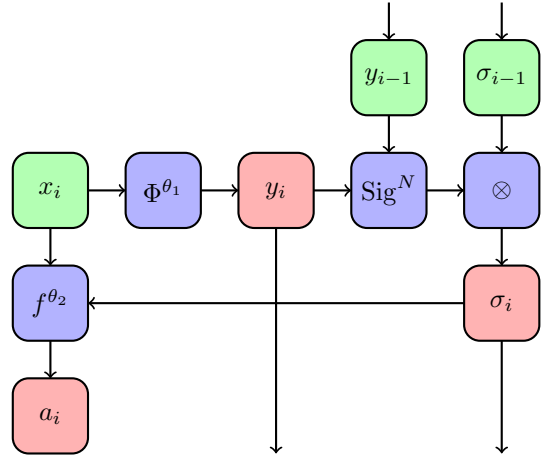
\begin{wrapfigure}[16]{r}{0pt}
		\begin{tikzpicture}
		\node (Position) at (-1.5,0){};
			\draw[rounded corners=\corners, black, fill=green!30, thick] (Position) rectangle ++(1, 1);
			\node[text width=4em, text centered] at ($(Position) + (0.5, 0.5)$) {$x_i$};
		\draw[black, thick, ->] (-0.5,0.5) -- (0,0.5);
		\node (Position) at (0,0){};
			\draw[rounded corners=\corners, black, fill=blue!30, thick] (Position) rectangle ++(1, 1);
			\node[text width=4em, text centered] at ($(Position) + (0.5, 0.5)$) {$\Phi^{\theta_1}$};
		\draw[black, thick, ->] (1,0.5) -- (1.5,0.5);
		\node (Position) at (1.5,0){};
			\draw[rounded corners=\corners, black, fill=red!30, thick] (Position) rectangle ++(1, 1);
			\node[text width=4em, text centered] at ($(Position) + (0.5, 0.5)$) {$y_i$};
		\draw[black, thick, ->] (2.5,0.5) -- (3,0.5);
		\draw[black, thick, ->] (2,0) -- (2,-3);
		\node (Position) at (3,0){};
			\draw[rounded corners=\corners, black, fill=blue!30, thick] (Position) rectangle ++(1, 1);
			\node[text width=4em, text centered] at ($(Position) + (0.5, 0.5)$) {$\mathrm{Sig}^N$};
		\draw[black, thick, ->] (3.5,3) -- (3.5,2.5);
		\draw[black, thick, ->] (3.5,1.5) -- (3.5,1);
		\node (Position) at (3,1.5){};
			\draw[rounded corners=\corners, black, fill=green!30, thick] (Position) rectangle ++(1, 1);
			\node[text width=4em, text centered] at ($(Position) + (0.5, 0.5)$) {$y_{i - 1}$};
		\draw[black, thick, ->] (4,0.5) -- (4.5,0.5);
		\node (Position) at (4.5,0){};
			\draw[rounded corners=\corners, black, fill=blue!30, thick] (Position) rectangle ++(1, 1);
			\node[text width=4em, text centered] at ($(Position) + (0.5, 0.5)$) {$\otimes$};
		\draw[black, thick, ->] (5,0) -- (5,-0.5);
		\draw[black, thick, ->] (5,-1.5) -- (5,-3);
		\node (Position) at (4.5,-1.5){};
			\draw[rounded corners=\corners, black, fill=red!30, thick] (Position) rectangle ++(1, 1);
			\node[text width=4em, text centered] at ($(Position) + (0.5, 0.5)$) {$\sigma_{i}$};
		\draw[black, thick, ->] (4.5, -1) -- (-0.5, -1);
		\draw[black, thick, ->] (5,1.5) -- (5,1);
		\draw[black, thick, ->] (5,3) -- (5,2.5);
		\node (Position) at (4.5,1.5){};
			\draw[rounded corners=\corners, black, fill=green!30, thick] (Position) rectangle ++(1, 1);
			\node[text width=4em, text centered] at ($(Position) + (0.5, 0.5)$) {$\sigma_{i - 1}$};
		\draw[black, thick, ->] (-1,0) -- (-1,-0.5);
		\node (Position) at (-1.5,-1.5){};
			\draw[rounded corners=\corners, black, fill=blue!30, thick] (Position) rectangle ++(1, 1);
			\node[text width=4em, text centered] at ($(Position) + (0.5, 0.5)$) {$f^{\theta_2}$};
		\draw[black, thick, ->] (-1,-1.5) -- (-1,-2);
		\node (Position) at (-1.5,-3){};
			\draw[rounded corners=\corners, black, fill=red!30, thick] (Position) rectangle ++(1, 1);
			\node[text width=4em, text centered] at ($(Position) + (0.5, 0.5)$) {$a_i$};	
		\end{tikzpicture}
		\caption{Agent architecture as a recurrent network. Trainable parameters: $\theta_1, \theta_2$.}
		\label{fig:rl-model-rnn}
\end{wrapfigure}

Furthermore  \setlength{\intextsep}{\backupintextsep} the $x_i$, $y_i$, $\sigma_i$ and $a_i$ may be collected into streams
\begin{align*}
(x_i)_i &\in \mathcal S(\mathbb R^d),\\
(y_i)_i &\in \mathcal S(\mathbb R^e),\\
(\sigma_i)_i &\in \mathcal S(\mathbb R^{(e^{N + 1} - 1)/(e - 1)}),\\
(a_i)_i &\in \mathcal S(\{\mathrm{actions}\}).
\end{align*}

In this way we may interpret this model as a generalisation of deep signature model: it has a single Neural-Lift-Signature block, with a skip connection across the whole block. The neural component is given by the neural network $\Phi^{\theta_1}$, which is stream-preserving as it operates pointwise, in the manner of equation \eqref{eq:pointwise}. The lift is the `expanding window' lift given by equation \eqref{eq:expanding}. Finally $f^{\theta_2}$ is another neural network, which is again pointwise and thus stream-preserving.

This interpretation of the model is demonstrated in Figure \ref{fig:rl-model-resnet}.

We test this model on a non-Markovian modification to the classical Mountain Car problem \cite{gym}, in which the agent receives only partial information: it is only given the car's position, and not its velocity.
We find that it is capable of learning how to solve the problem within a set number of episodes, whilst a comparable RNN architecture fails to do so. The reinforcement learning technique used was Deep Q Learning \cite{mnih2015human} with the specified models performing function approximation on $Q$. Both models were chosen to have comparable numbers of parameters. Further implementation details can be found in Appendix \ref{appendix:implementation}.

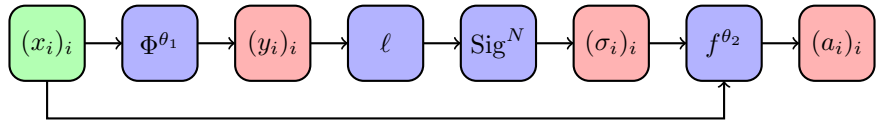
\begin{figure}[h]
	\centering
	\begin{tikzpicture}
		\node (Position) at (-1.5,0){};
			\draw[rounded corners=\corners, black, fill=green!30, thick] (Position) rectangle ++(1, 1);
			\node[text width=4em, text centered] at ($(Position) + (0.5, 0.5)$) {$(x_i)_i$};
		\draw[black, thick, ->] (-0.5,0.5) -- (0,0.5);
		\draw[black, thick, ->] (-1,0) -- (-1,-0.5) -- (8, -0.5) -- (8, 0);	
		\node (Position) at (0,0){};
			\draw[rounded corners=\corners, black, fill=blue!30, thick] (Position) rectangle ++(1, 1);
			\node[text width=4em, text centered] at ($(Position) + (0.5, 0.5)$) {$\Phi^{\theta_1}$};
		\draw[black, thick, ->] (1,0.5) -- (1.5,0.5);
		\node (Position) at (1.5,0){};
			\draw[rounded corners=\corners, black, fill=red!30, thick] (Position) rectangle ++(1, 1);
			\node[text width=4em, text centered] at ($(Position) + (0.5, 0.5)$) {$(y_i)_i$};
		\draw[black, thick, ->] (2.5,0.5) -- (3,0.5);
		\node (Position) at (3,0){};
			\draw[rounded corners=\corners, black, fill=blue!30, thick] (Position) rectangle ++(1, 1);
			\node[text width=4em, text centered] at ($(Position) + (0.5, 0.5)$) {$\ell$};
		\draw[black, thick, ->] (4,0.5) -- (4.5,0.5);
		\node (Position) at (4.5,0){};
			\draw[rounded corners=\corners, black, fill=blue!30, thick] (Position) rectangle ++(1, 1);
			\node[text width=4em, text centered] at ($(Position) + (0.5, 0.5)$) {$\mathrm{Sig}^N$};
		\draw[black, thick, ->] (5.5,0.5) -- (6,0.5);
		\node (Position) at (6,0){};
			\draw[rounded corners=\corners, black, fill=red!30, thick] (Position) rectangle ++(1, 1);
			\node[text width=4em, text centered] at ($(Position) + (0.5, 0.5)$) {$(\sigma_{i})_i$};
		\draw[black, thick, ->] (7,0.5) -- (7.5,0.5);
		\node (Position) at (7.5,0){};
			\draw[rounded corners=\corners, black, fill=blue!30, thick] (Position) rectangle ++(1, 1);
			\node[text width=4em, text centered] at ($(Position) + (0.5, 0.5)$) {$f^{\theta_2}$};
		\draw[black, thick, ->] (8.5,0.5) -- (9,0.5);
		\node (Position) at (9,0){};
			\draw[rounded corners=\corners, black, fill=red!30, thick] (Position) rectangle ++(1, 1);
			\node[text width=4em, text centered] at ($(Position) + (0.5, 0.5)$) {$(a_i)_i$};	
		\end{tikzpicture}
		\caption{Agent architecture as a residual network. Trainable parameters: $\theta_1, \theta_2$. The lift $\ell$ is the `expanding window' lift of equation \eqref{eq:expanding}.}
		\label{fig:rl-model-resnet}
\end{figure}

\section{Conclusion}
There is a strong corpus of theory motivating the use of the signature transform as a tool to understand streams of data. Meanwhile neural networks have enjoyed great empirical success. It is thus desirable to bring them together; in this paper we have described how this may be done in a general fashion, and have provided examples of how this principle may be used in a variety of domains.

There are two key contributions. First, we discuss stream-preserving neural networks, which are what allow for using signature transforms deeper within a network, rather than as just a feature transformation. Second, we discuss lifts, which is what allows for the use of multiple signature transforms. In this way we have significantly extended the use of the signature transform in machine learning: rather than limiting its usage to data preprocessing, we demonstrate how the signature transform, as a univeral nonlinearity, may be used as a pooling layer within a neural network.

	\subsubsection*{Acknowledgements}
	PB was supported by the EPSRC grant EP/R513295/1. PK was supported by the EPSRC grant EP/L015811/1. PK, IPA, CS, TL were supported by the Alan Turing Institute under the EPSRC grant EP/N510129/1.
	
	\small
	\bibliography{references} 
	\bibliographystyle{ieeetr}
	
	\normalsize
	\newpage 
	\appendix
\section{A brief overview of signatures}\label{appendix:sigprop}
This appendix is split into three subsections. The first subsection discusses the definition and properties of the signature transform on path space, which is the mathematically natural way to approach things. The second subsection goes on to adapt the signature transform to the space of streams of data.

In the third subsection we discuss how to compute the signature transform. In particular we will see that whilst the signature transform of a path may look somewhat unfriendly to compute, it will (fortunately!) turn out that the signature transform may be efficiently computed in the special case that its input is piecewise linear, which is how a stream of data is interpreted.

\subsection{Signatures of paths}
We begin with the definition of the signature. Note that this definition is written in a slightly different format to that of Definition \ref{sig-definition}, as the more traditional (if somewhat unfriendly-looking) notation of stochastic calculus is used; however the mathematical meaning is the same.
\begin{definition}[\cite{lyons1998differential}]\label{def:sigstochasticnotation}
	Let $a, b \in \mathbb R$, and $X = (X^1, \ldots, X^d) \colon [a, b] \to \mathbb R^d$ be a continuous piecewise smooth path.  The signature of $X$ is then defined as the collection of iterated integrals
	\begin{align*}
	\mathrm{Sig}(X) &= \left( \underset{a < t_1 < \cdots < t_k < b}{\int\cdots\int} \mathrm dX_{t_1} \otimes \cdots \otimes  \mathrm dX_{t_k} \right)_{k \geq 0}\\
	&= \left(\left( \underset{a < t_1 < \cdots < t_k < b}{\int\cdots\int} \mathrm dX^{i_1}_{t_1} \cdots \mathrm dX^{i_k}_{t_k} \right)_{1 \leq i_1, \ldots, i_k \leq d}\right)_{k\geq 0},
	\end{align*}
	where $\otimes$ denotes the tensor product, $ \mathrm dX_t = \frac{\mathrm dX_t}{\mathrm dt}{\mathrm dt} $, and the $k = 0$ term is taken to be $1 \in \mathbb R$.
	
	The truncated signature of depth $N$ of $X$ is defined as
	\begin{equation*}
	\mathrm{Sig}^N(X)= \left( \underset{a < t_1 < \cdots < t_k < b}{\int\cdots\int} \mathrm dX_{t_1} \otimes \cdots \otimes  \mathrm dX_{t_k} \right)_{0 \leq k \leq N}.
	\end{equation*}
\end{definition}
The signature may in fact be defined much more generally, on paths of merely bounded variation, see \cite{lyons1998differential, FritzVictoir10}, but the above definition suffices for our purposes. This broader theory is also the reason for the notation $\mathrm dX_t$, which may be made sense of even when $X$ is not differentiable.

\begin{example} \label{ex:sigincrement}
	Suppose $X \colon [a, b] \to \mathbb R^d$ is the linear interpolation of two points $ x,y \in \mathbb R^d $, so that $ X_t = x + \frac{t-a}{b-a}(y-x) $ . Then its signature is just the collection of powers of its total increment:
	\begin{equation*}
	\mathrm{Sig}(X) = \left(1, y-x, \frac{1}{2} (y-x)^{\otimes 2}, \frac{1}{6} (y-x)^{\otimes 3}, \ldots, \frac{1}{k!} (y-x)^{\otimes k}, \ldots \right).
	\end{equation*}
	Which is independent of $a$, $b$.
\end{example}

\begin{definition}
Given a path $X \colon [a, b] \to \mathbb R^d$, we define the corresponding \emph{time-augmented path} by $\widehat{X}_t = (t, X_t)$, which is a path in $\mathbb R^{d + 1}$.
\end{definition}

The signature exhibits four key properties that makes its use attractive when dealing with path-like data. First, a path is essentially defined by its signature. This means that essentially no information is lost when applying the signature transform.
\begin{proposition}[Uniqueness of signature \cite{hambly2010uniqueness}]\label{prop:uniqueness}
	Let $X \colon [a, b] \to \mathbb R^d$ be a continuous piecewise smooth path. Then $\mathrm{Sig}(\widehat{X})$ uniquely determines $X$ up to translation.
\end{proposition}

Next, the terms of the signature decay in size factorially.
\begin{proposition}[{Factorial decay \cite[Lemma 2.1.1]{lyons1998differential}}]\label{prop:factorialdecay}
	Let $X \colon [a,b]\to \mathbb R^d$ be a continuous piecewise smooth path. Then
	\begin{equation*}
	\left \lVert \underset{a < t_1 < \cdots < t_k < b}{\int\cdots\int} \mathrm{d}X_{t_1} \otimes \cdots \otimes  \mathrm{d}X_{t_k} \right \rVert \leq \dfrac{C(X)^k}{k!},
	\end{equation*} 
	where $C(X)$ is a constant depending on $X$ and $ \lVert \,\cdot\, \rVert $ is any tensor norm on $ (\mathbb R^d)^{\otimes k} $.
\end{proposition}

Third, functions of the path are approximately linear on the signature. In some sense the signature may be thought of as a `universal nonlinearity' on paths.
\begin{proposition}[Universal nonlinearity \cite{perezarribas2018}]\label{prop:universal}
	Let $F$ be a real-valued continuous function on continuous piecewise smooth paths in $\mathbb R^d$ and let $\mathcal K $ be a compact set of such paths.\footnote{Of course the definition of both continuity and compactness depend on the topology of the set of paths. See \cite{lyons2014rough} for details.} Furthermore assume that $X_0 = 0$ for all $X \in \mathcal K$. (To remove the transalation invariance.) Let $\varepsilon >0$. Then there exists a linear functional $L$ such that for all $X \in \mathcal K$,
	\begin{equation*}\left \vert F(X) - L(\mathrm{Sig}(\widehat{X})) \right\vert < \varepsilon
	\end{equation*}
\end{proposition}

Finally, the signature is invariant to time reparameterisations.
\begin{proposition}[Invariance to time reparameterisations  \cite{lyons1998differential}] \label{prop:invariance}
Let $X \colon [0, 1] \to \mathbb R^d$ be a continuous piecewise smooth path. Let $\psi \colon [0, 1] \to [0, 1]$ be continuously differentiable, increasing, and surjective. Then $\mathrm{Sig}(X) = \mathrm{Sig}(X \circ \psi)$.
\end{proposition}
Thus the signature encodes the \emph{order} in which data arrives without caring precisely \emph{when} it arrives; it is essentially factoring out the infinite-dimensional group of time reparameterisations. For example, consider the scenario of recording the movement of a pen as it draws a character on a piece of paper. Then the signature of the stream of data is invariant to the speed at which the character was drawn.

\begin{remark}
There is an interesting interplay between Proposition \ref{prop:universal} and Proposition \ref{prop:invariance}. If one desires invariance to time reparameterisations, as in the example of a pen drawing a character, then computing the signature of just $X$ rather than $\widehat{X}$ will ensure by Proposition \ref{prop:invariance} that this invariance is present. If one does not desire invariance to time reparameterisations, then using the time-augmented path $\widehat{X}$ is what ensures that parameterisation-dependent functions may still be learned. This essentially corresponds to the difference between $\widehat{X \circ \psi}$ and $\widehat{X} \circ \psi$.
\end{remark}

\subsection{Signatures of streams of data}
We interpret a stream of data as a discretisation of a path.
\begin{definition}
	The space of streams of data is defined as
	\begin{equation*}
	\mathcal S(\mathbb R^d) = \{ \mathbf x=(x_1, \ldots, x_n) : x_i \in \mathbb R^d, n \in \mathbb N\}.
	\end{equation*}
	Given $\mathbf x=(x_1, \ldots, x_n) \in \mathcal S(\mathbb R^d)$, the integer $n$ is called the length of $\mathbf x$. Furthermore for $ a , b \in \mathbb R $ such that $ a < b $, fix
	\begin{equation*}
	a = u_1 < u_2 < \cdots < u_{n - 1} < u_n = b.
	\end{equation*}
	Let $X = (X^1, \ldots, X^d) \colon [a,b] \to \mathbb R^d$ be continuous such that $X_{u_i} = x_i$ for all $i$, and linear on the intervals in between. Then $X$ is called a linear interpolation of $\mathbf x$.
\end{definition}

\begin{definition}\label{def:sig}
	Let $\mathbf x=(x_1, \ldots, x_n)\in \mathcal S(\mathbb R^d)$ be a stream of data. Let $X$ be a linear interpolation of $\mathbf x$. Then the signature of $\mathbf x$ is defined as
	\begin{equation*}
	\mathrm{Sig}(\mathbf x) = \mathrm{Sig}(X)
	\end{equation*}
	and the (truncated) signature of order $N$ of $\mathbf x$ is defined as
	\begin{equation*}
	\mathrm{Sig}^N(\mathbf x) = \mathrm{Sig}^N(X).
	\end{equation*}
\end{definition}
\emph{A priori} this definition of the signature of a stream of data depends on the choice of linear interpolation. (That is, the speed at which one traverses the gap between the $x_i$.) However, it turns out that Definition \ref{def:sig} is well-defined and independent of this choice, by Proposition \ref{prop:invariance}. See \cite[Lemma 2.12]{lyons1998differential}.

\begin{remark}
	Let $\mathbf x = (x_1, \ldots, x_n) \in \mathcal S(\mathbb R^d)$ be a stream of data of length $n$ in $ \mathbb{R}^d $. Then $\mathrm{Sig}^N(\mathbf x)$ has
	\begin{equation*}
	\sum_{k = 0}^N d^k = \frac{d^{N+1} - 1}{d-1}
	\end{equation*}
	components. In particular, the number of components does not depend on $n$; the truncated signature maps the infinite-dimensional space of streams of data $\mathcal S(\mathbb R^d)$ into a finite-dimensional space of dimension $(d^{N+1} - 1)/(d-1)$. Thus the signature is an excellent way to tackle long streams of data, or streams of variable length, or streams for which certain data is missing.
\end{remark}

\subsection{Computing the signature}\label{subsection:computation}

Observe that the signature is defined as a sequence where the zeroth term is $1 \in (\mathbb {R}^d)^{\otimes 0} = \mathbb {R} $, the first term belongs to $\mathbb {R}^d$, the second term belongs to $ \mathbb {R}^d \otimes \mathbb {R}^d $ (that is, the space of matrices of size $d\times d$), the third term belongs to $ \mathbb {R}^d \otimes \mathbb {R}^d \otimes \mathbb {R}^d $ (that is, the space of tensors of shape $(d, d, d)$), and the $k$th term belongs to $(\mathbb {R}^d)^{\otimes k} = \mathbb {R}^d \otimes \cdots \otimes \mathbb {R}^d$, $ k $ times (that is, the space of tensors of shape $(d, \ldots, d)$, $k$ times). With this description, the signature naturally takes values in the tensor algebra:
\begin{definition}
	The tensor algebra of $\mathbb R^d$ is defined as
	\begin{equation*}
	T((\mathbb R^d)) = \prod_{k = 0}^\infty (\mathbb {R}^d)^{\otimes k}.
	\end{equation*}
\end{definition}

	The tensor product $\otimes$ is typically defined between two tensors, taking a tensor of shape $(a_1, \ldots, a_n)$ and a tensor of shape $(b_1, \ldots, b_m)$ to a tensor of shape $(a_1, \ldots, a_n, b_1, \ldots, b_m)$. For example, in the special case that these two tensors are of shapes $(a_1)$, $(b_1)$, so that they are vectors, then the tensor product is what is referred to as the \emph{outer product}.

\begin{definition}\label{def:tensor-product}
	When extended by bilinearity, the tensor product defines a multiplication on $ T((\mathbb R^d)) $. For $A = (A_0, A_1, \ldots) \in T((\mathbb R^d))$ and $B = (B_0, B_1, \ldots) \in T((\mathbb R^d))$, then $A \otimes B \in T((\mathbb R^d))$ can be seen to be
	\begin{equation*}
	A \otimes B = \left(\sum_{j = 0}^k A_j \otimes B_{k - j}\right)_{k \geq 0}.
	\end{equation*}
\end{definition}

A fundamental insight of Chen is that concatenation of paths corresponds to tensor multiplication of their signatures. The following relation is known as \textit{Chen's identity}.
\begin{proposition}[{Chen's identity, \cite[Theorem 2.12]{lyons1998differential}}]
	Let $X \colon [a,b] \to \mathbb R^d$ and $Y \colon [a,b] \to \mathbb R^d$ be two continuous piecewise smooth paths such that $X_b = Y_a$. Define their concatenation $X\ast Y$ as
	\begin{empheq}[left={(X\ast Y)_t = \empheqlbrace}]{align*}
	X_{2t-a} &\quad\text{for}\quad a \leq t < \frac{a+b}{2},\\
	Y_{2t -b}  &\quad\text{for}\quad \frac{a+b}{2} \leq t \leq b.
	\end{empheq}
	Then
	\begin{equation*}
	\mathrm{Sig}(X\ast Y) = \mathrm{Sig}(X) \otimes \mathrm{Sig}(Y).
	\end{equation*}
\end{proposition}

Equipped with Chen's identity, the signature of a stream is straightforward to compute explicitly.
\begin{proposition} \label{prop:chenefficient}
	Let $\mathbf x = (x_1, \ldots, x_n)\in \mathcal S(\mathbb R^d)$ be a stream of data. Then,
	\begin{equation*}
	\mathrm{Sig}(\mathbf x) = \exp(x_2 - x_1) \otimes \exp(x_3 - x_2) \otimes \cdots \otimes \exp(x_n - x_{n-1}),
	\end{equation*}
	where
	\[\exp(x) = \left(\frac{x^{\otimes k}}{k!}\right)_{k \geq 0} \in T((\mathbb R^d)).\]
\end{proposition}
\begin{proof}
	It is easy to check that if $\mathbf x = (x_1, x_2)\in \mathcal S(\mathbb R^d)$ is a stream of data of length 2 then the signature of $\mathbf x$ is given by $\exp(x_2 - x_1)$, as in Example \ref{ex:sigincrement}. So given a stream of data $\mathbf x = (x_1, \ldots, x_n)\in \mathcal S(\mathbb R^d)$ of length $n \geq 2$, iteratively applying Chen's identity yields the result.
\end{proof}

\begin{remark} \label{re:efficiency}
	Chen's identity implies that computing the signature of an incoming stream of data is efficient. Indeed, suppose one has obtained a stream of data $\mathbf x\in \mathcal S(\mathbb R^d)$ and computed its signature. Suppose that after some time more data has arrived, $\mathbf y \in \mathcal S(\mathbb R^d)$. In order to compute the signature of the whole signal one only needs to compute the signature of the new piece of information, and tensor product it with the already-computed signature.
\end{remark}

\begin{remark} Computing signatures in the manner described here involves only normal tensor operations, so it may be backpropagated through in the usual way. Recall that signatures are fundamentally defined on path space; backpropagating corresponds to determining the perturbation of the signature when perturbing its input with white noise. However one of the insights of \emph{rough path theory} \cite{lyons1998differential} is that a path needs more than just its pointwise values to be fully determined. The most common example of this arises in stochastic calculus, where one has to make a choice between It{\^o} and Stratonovich integration. Until such a choice is made, one cannot define a notion of integrals of the path. In general, for sufficiently rough paths, one has to \emph{define} what the integrals of a path are: essentially the path is defined by its signature, rather than the other way around. In such a framework it is not clear what the correct notion of perturbations of path space are, and this remains a direction for future work.
\end{remark}

\section{Implementation Details} \label{appendix:implementation}
	All experimental models were trained using the Adam \cite{kingma2015} optimiser as implemented by PyTorch \cite{paszke2017automatic}, which was the framework used to implement the models. Signature calculations were performed with the \texttt{iisignature} package \cite{iisignature} (as the Signatory project \cite{signatory} mentioned elsewhere in this paper had not yet been developed). All activation functions were taken to be the ReLU. Computations were performed on two computers. One was equipped with two Tesla K40m GPUs. The second was equipped with two GeForce RTX 2080 Ti GPUs and two Quadro GP100 GPUs.

In each of the following sections, the notation is the same as the notation used in the corresponding section of the main document.

\subsection{A generative model for a stochastic process}
The training dataset was given by 1024 realisations of an Ornstein--Uhlenbeck process, and the test set was of the same size, each sampled at 100 points of $[0, 1]$. No minibatching was used. The model was trained for 500 epochs. 

The layer $\Phi^{\theta_1}$ operated pointwise on the stream of time-augmented Brownian motion $\mathbf{B} = (t_i, B_{t_i})_i \in \mathcal S(\mathbb R^2)$, and was taken to be a neural network with 2 output neurons and 2 hidden layers of 8 neurons. Furthermore it kept the original stream; thus
\begin{equation*}
\Phi^{\theta_1}(\mathbf{B}) = (t_i, B_{t_i}, \phi_1^{\theta_1}(t_i, B_{t_i}), \phi_2^{\theta_1}(t_i, B_{t_i})) \in \mathcal S(\mathbb R^4)
\end{equation*}
for some learned $\phi^{\theta_1}_1, \phi^{\theta_1}_2$. The lift was the `expanding window' described in equation \eqref{eq:expanding}. The signature in the generator was truncated at $N=3$ (giving 84 scalar nonconstant terms) The layer $f^{\theta_2}$ operated pointwise on the stream of signatures, and was a simple linear map down to a scalar value (the value of the generated process at that time step). The signature in the discriminator was truncated at $M=4$.

Some hyperparameter searching was necessary to obtain good results. The search was not done according to any formal scheme. It seemed that if $\Phi^{\theta_1}$ was sufficiently simple and not did not keep the original stream then the training would easily get trapped in a bad local minima, and the generated process would be visually distinct from the Ornstein--Uhlenbeck process.

\subsection{Supervised learning with fractional Brownian motion}
The training set featured 600 samples whilst the test set featured 100 samples, each of an instance of fractional Brownian motion sampled at 300 time steps of $[0, 1]$, with Hurst parameters in the range $[0.2, 0.8]$. These were split up into batches of 128 samples, so the last batch is slightly smaller than the others, and every model trained for 100 epochs. The loss function was taken to be mean squared error (MSE).

There was no hyperparameter searching except to require that all models should have approximately the same number of parameters; in all cases the results represent a model whose hyperparameters have not been fine-tuned to the task at hand.

All models used a sigmoid as a final nonlinearity, so as to map in to $(0, 1)$.

The differing sizes of layers between models (whilst keeping roughly the same overall parameter count) is usually because of the varying size of the input to the model. Some models take all of the raw data, some models use signatures, and some models take expanding or sliding windows of the data in a manner akin to equations \eqref{eq:expanding} and \eqref{eq:sliding}.

The Feedforward model was a simple neural network with 3 hidden layers of 16 neurons each.

The Neural-Sig model -- which is essentially the same model as the Feedforward model, except that the data has the signature applied as feature transformation first -- featured hidden layers of sizes 64, 64, 32, 32, 16, 16 respectively.

The RNN model is two recurrent neural networks, the first comprised of dense layers of sizes 64, 64, 32 and output size 6, and the second comprised of dense layers of size 32, 32, 32, and output size 5. The first network sweeps across the input data relatively slowly, with a stride of 2, whilst the second network sweeps across the result of the first network more quickly, with a stride of 4. In this way it may capture information from the input data at multiple timescales; part of the challenge of fractional Brownian motion is the existence of long-range dependencies \cite{embrechts2009selfsimilar}.

The LSTM and GRU models both featured two recurrent layers each of size 32, and swept across the raw data with a stride of 1.

DeepSigNet featured a single Neural-Lift-Signature block, where the neural component was given by a single convolutional layer with 3 channels and kernel size 3, the lift was the trivial lift of equation \eqref{eq:nolift}, and the signature was truncated as $N=3$. The neural component also preserved the original time-augmented stream of data, so that in some sense the neural component has 3 extra channels corresponding to time and value. On top of this a feedforward neural network with 5 hidden layers of size 32 was placed. Thus this model is very similar to the Neural-Sig model, except that a small learnable transformation was allowed before the signature. The difference in their performance highlights the value of learning a transformation before using the signature. (Without which the Neural-Sig model is merely outperformed by some non-signature based models.)

DeeperSigNet featured three Neural-Lift-Signature blocks. The neural component of the first block was a small feedforward network with 2 hidden layers of size 16 and an output layer of size 3, swept across the length of the stream; its kernel size (how many time-value pairs of the stream it saw at once) was 4. The original time-augmented stream of data was also preserved by the neural component. The neural components of the other two blocks were recurrent neural networks, featuring 2 hidden layers of 16 neurons each. The lifts were in every case expanding windows as in equation \eqref{eq:expanding}. On top of this another recurrent neural network was placed, and the value of its final hidden state used as the output. This final network used 2 hidden layers of 16 neurons each.

\subsection{Non-Markovian deep reinforcement learning}
We used the implementation of the Mountain Car problem implemented by the OpenAI Gym \cite{gym}, modified to return only the car's position. Each episode was run for 300 steps, and each model was given 2000 episodes in which to learn. The reward function was given by the car's position, in the range $(-1.2, 0.6)$, with a bonus $+1$ on reaching the goal. At each step the car could drive its engine left, right, or not use it at all. This problem was chosen for its ease of implementation.

The sizes of the models were chosen to ensure that they both had roughly the same number of scalar parameters. Within this specification, there was a small amount of hyperparameter searching. This was done in an \emph{ad hoc} manner, for both models, varying the number of layers and the numbers of neurons in each layer, around the values that were eventually used. The eventual values chosen for the deep signature model were selected as the ones giving the best results for the deep signature model. The eventual values for the RNN were selected to give roughly the same parameter count as the deep signature model, as no RNN model achieved any appreciable success.

The deep signature model was as described in Section \ref{section:rl}, with the first network $\Phi^{\theta_1}$ applying a learned linear transformation with output dimension 2. Furthermore it kept the original time-augmented stream, so that
\begin{equation*}
y_i = \Phi^{\theta_1}(x_i) = (t_i, x_i, \phi_1^{\theta_1}(t_i, x_i), \phi_2^{\theta_1}(t_i, x_i)) \in \mathcal S(\mathbb R^4),
\end{equation*}
where $\phi_1^{\theta_1}$ and $\phi_2^{\theta_2}$ are learned linear functions. The signature was truncated at $N=3$. The second network $f^{\theta_2}$ was comprised of a single hidden layer of 64 neurons, followed by an output layer of 3 neurons, corresponding to the three possible actions. The action with the greatest value was the action selected. This model had a total of 5769 scalar parameters.

The RNN model featured 3 recurrent layers each of size 32, followed by an output layer of 3 neurons, corresponding to the three possible actions. The action with the greatest value was the action selected. This model had a total of 5475 scalar parameters.


The reinforcement learning technique used was Deep Q Learning \cite{mnih2015human, watkins1989learning}, to effectively transform the task into a supervised learning problem, with each of the specified models performing function approximation on $Q$. Actions were chosen in an $\varepsilon$-greedy manner, with $\varepsilon=0.2$. The discount factor was given by $\gamma=0.99$.

The deep signature model achieved success, and would learn to consistently solve the problem at around 1500 episodes. The RNN failed to achieved success within 2000 episodes on any test run. 3 test runs were performed.
\end{document}